\documentclass{article} % For LaTeX2e
\usepackage{iclr2022_conference,times}

\usepackage[utf8]{inputenc}
\usepackage[T1]{fontenc}
\usepackage{textcomp}

\definecolor{linkcolor}{RGB}{74, 102, 146}
\usepackage[colorlinks=true,allcolors=linkcolor,pageanchor=true,plainpages=false,pdfpagelabels,bookmarks,bookmarksnumbered]{hyperref}

\usepackage{amsmath,amsfonts,bm}

\def\eqref#1{equation~\ref{#1}}
\def\1{\bm{1}}

\DeclareMathAlphabet{\mathsfit}{\encodingdefault}{\sfdefault}{m}{sl}
\SetMathAlphabet{\mathsfit}{bold}{\encodingdefault}{\sfdefault}{bx}{n}

\usepackage{graphicx}
\usepackage{tikz}
\usepackage{hyperref}
\usepackage{url}
\usepackage{subcaption}

\usepackage{amsthm}

\newtheorem{remark}{Remark}
\newtheorem{proposition}{Proposition}
\newtheorem{definition}{Definition}

\newtheorem{theorem}{Theorem}

\usepackage{xspace}
\makeatletter
\DeclareRobustCommand\onedot{\futurelet\@let@token\@onedot}
\def\@onedot{\ifx\@let@token.\else.\null\fi\xspace}
\newcommand{\eg}{\emph{e.g}\onedot}
\newcommand{\ie}{\emph{i.e}\onedot}
\makeatother

\usepackage{algorithm, algpseudocode}
\usepackage{amssymb}

\usepackage{ bbold }
\DeclareMathOperator\supp{supp}

\usepackage[nameinlink]{cleveref}
\Crefname{section}{Sect.}{Sects.}
\Crefname{appendix}{App.}{Apps.}
\Crefname{proposition}{Prop.}{Props.}

\title{Cross-Domain Imitation Learning\\ via Optimal Transport}

\author{
  Arnaud Fickinger$^{1 3}$\thanks{\texttt{arnaud.fickinger@berkeley.edu, arnaudfickinger@fb.com}}\quad
  Samuel Cohen$^{2 3}$\quad
  Stuart Russell$^1$\quad
  Brandon Amos$^3$ \\
$^1$Berkeley AI Research\quad
$^2$University College London\quad
$^3$Facebook AI
}

\definecolor{bdacolor}{RGB}{168, 141, 201}

\newcommand{\xxnote}[3]{}
\ifx\hidenotes\undefined
  \renewcommand{\xxnote}[3]{\color{#2}{#1: #3}}
\fi

\iclrfinalcopy
\begin{document}

\maketitle

\begin{abstract}
Cross-domain imitation learning studies how to leverage expert
demonstrations of one agent to train an imitation agent with a
different embodiment or morphology. Comparing trajectories and
stationary distributions between the expert and imitation agents is
challenging because they live on different systems that may not even
have the same dimensionality.
We propose \emph{Gromov-Wasserstein Imitation Learning (GWIL)}, a method for cross-domain imitation that uses the Gromov-Wasserstein distance 
to align and compare states between the different
spaces of the agents.
Our theory formally characterizes the scenarios where GWIL preserves optimality, revealing its possibilities and limitations.
We demonstrate the effectiveness of GWIL in non-trivial continuous
control domains ranging from simple
rigid transformation of the expert domain to arbitrary
transformation of the state-action space. \footnote{Project site with videos and code: \href{https://arnaudfickinger.github.io/gwil/}{https://arnaudfickinger.github.io/gwil/}}

\end{abstract}

\section{Introduction}
Reinforcement learning (RL) methods have attained impressive results across a number of domains, e.g., \citet{berner2019dota,kober2013reinforcement,levine2016end,vinyals2019grandmaster}.
However, the effectiveness of current RL method is heavily correlated
to the quality of the training reward.
Yet for many real-world tasks, designing dense and
informative rewards require significant engineering effort.
To alleviate this effort, imitation learning (IL)
proposes to learn directly from expert demonstrations.
Most current IL approaches can be applied solely to the simplest
setting where the expert and the agent share the same
embodiment and transition dynamics that live in the same
state and action spaces.
In particular, these approaches require expert demonstrations
from the agent domain.
Therefore, we might reconsider the utility of
IL as it seems to only move the problem,
from designing informative rewards to providing expert
demonstrations, rather than solving it.
However, if we relax the constraining setting of
current IL methods, then natural imitation scenarios that genuinely
alleviate engineering effort appear.
Indeed, not requiring the same dynamics
would enable agents to imitate humans and robots with
different morphologies,
hence widely enlarging the applicability of IL and alleviating
the need for in-domain expert demonstrations.

This relaxed setting where the expert demonstrations comes from another
domain has emerged as a budding area with more realistic assumptions
\citep{gupta2017learning,liu2019state,sermanet2018time,kim2020domain,raychaudhuri2021cross}
that we will refer to as
\emph{Cross-Domain Imitation Learning.}
A common strategy of these works is to learn a mapping
between the expert and agent domains.
To do so, they require access to proxy tasks where both the
expert and the agent act optimally in there respective domains.
Under some structural assumptions, the learned map enables to
transform a trajectory in the expert domain
into the agent domain while preserving the optimality.
Although these methods indeed relax the typical setting of IL,
requiring proxy tasks heavily restrict the applicability of
Cross-Domain IL.
For example, it rules out imitating an expert never seen before
as well as transferring to a new robot.

In this paper, we relax the assumptions of Cross-Domain IL and propose
a benchmark and method that do not need access to proxy tasks.
To do so, we depart from the point of view taken by previous work and %[explicitly]
formalize Cross-Domain IL as an optimal transport problem.
We propose a method, that we call \emph{Gromov Wasserstein Imitation Learning
(GWIL)}, that uses the Gromov-Wasserstein distance to solve the benchmark.
We formally characterize the scenario where GWIL preserves
optimality (\cref{the:1}), revealing the possibilities and limitations.
The construction of our proxy rewards to optimize optimal
transport quantities using RL generalizes previous work that assumes
uniform occupancy measures \citep{dadashi2020primal, papagiannis2020imitation}
and is of independent interest.
Our experiments show that GWIL learns optimal behaviors with
a single demonstration from another domain without any proxy tasks
in non-trivial continuous control settings.

\section{Related Work}

\begin{figure}[t]
  \centering
  \includegraphics[width=1\textwidth]{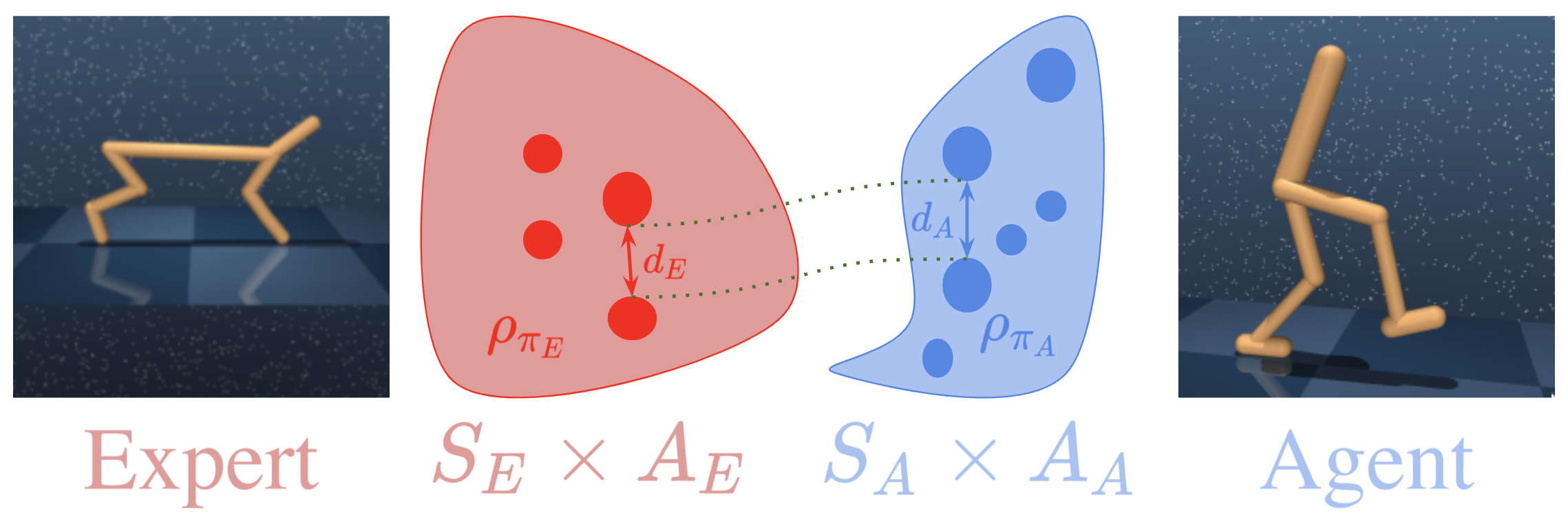}
  \caption{
    The Gromov-Wasserstein distance enables us to compare the
    stationary state-action distributions of two agents with different
    dynamics and state-action spaces. We use it as a pseudo-reward
    for cross-domain imitation learning.
    % \bda{SEMI-RESOLVED Figure is looking good! One last polish: add titles
    %   saying ``Expert'' and ``Agent'' on each side of the figure.
    %   And could use agents with a white background.
    % }
  }
  \label{fig:overview}
\end{figure}

\begin{figure}[t]
  \centering
  \includegraphics[width=\textwidth]{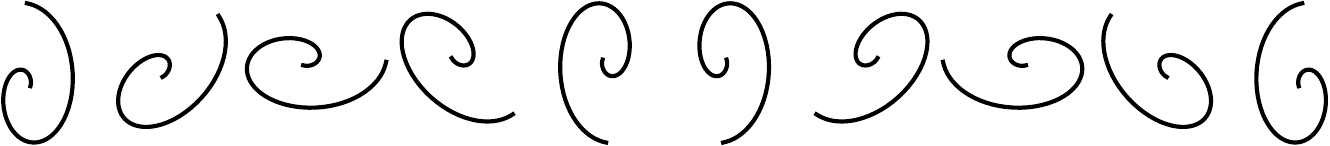}
  \vspace{-5mm}
  \caption{
    Isomorphic policies (\cref{def:2}) have the same pairwise distances
    within the state-action space of the stationary distributions.
    In Euclidean spaces, isometric transformations preserve these
    pairwise distances and include rotations, translations, and reflections.
  }
  \label{fig:isomorphic-policies}
\end{figure}

\textbf{Imitation learning.}
An early approach to IL is Behavioral Cloning
\citep{Pomerleau1988bc1,Pomerleau1991bc2} which amounts to training a
classifier or regressor via supervised learning to replicate the
expert’s demonstration.
Another key approach is Inverse Reinforcement Learning \citep{ngrussel_irl,abbeel2004apprenticeship,abbeel2010autonomous},
which aims at learning a reward function under which the
observed demonstration is optimal and can then be used to train a
agent via RL.
To bypass the need to learn the expert’s reward
function, \citet{Ho2016GAIL} show that IRL is a dual of an occupancy
measure matching problem and propose an adversarial objective whose
optimization approximately recover the expert’s state-action occupancy
measure, and a practical algorithm that uses a generative adversarial
network \citep{goodfellow2014generative}.
While a number of recent
work aims at improving this algorithm relative to the training
instability caused by the minimax optimization, Primal Wasserstein
Imitation Learning (PWIL) \citep{dadashi2020primal} and
Sinkhorn Imitation Learning (SIL) \citep{papagiannis2020imitation}
view IL as an optimal transport problem
between occupancy measures to completely eliminate the minimax
objective and outperforms adversarial methods in terms of
sample efficiency.
\citet{heess2017emergence,peng2018deepmimic,zhu2018reinforcement,aytar2018playing}
scale imitation learning to complex human-like
locomotion and game behavior in non-trivial settings.
Our work is an extension of
\citet{dadashi2020primal,papagiannis2020imitation}
from the Wasserstein to the Gromov-Wasserstein setting.
This takes us beyond limitation that the expert
and imitator are in the same domain and
into the cross-domain setting between agents that
live in different spaces.

\textbf{Transfer learning across domains and morphologies.}
Work transferring knowledge between different domains in
RL typically learns a mapping between the state and action spaces.
\citet{ammar2015unsupervised} use unsupervised manifold
alignment to find a linear map between states that have similar local
geometry but assume access to hand-crafted features. More recent work
in transfer learning across viewpoint and embodiment mismatch learn a
state mapping without handcrafted features but assume access to paired
and time-aligned demonstration from both domains
\citep{gupta2017learning,liu2018imitation,sermanet2018time}.
Furthermore, \citet{kim2020domain,raychaudhuri2021cross} propose methods
to learn a state mapping from unpaired and unaligned tasks. All these
methods require proxy tasks, \ie a set of pairs of expert
demonstrations from both domains, which limit the applicability of
these methods to real-world settings.
\citet{stadie2017third} have proposed to combine adversarial
learning and domain confusion to learn a policy in the agent’s domain
without proxy tasks but their method only works in the case of small
viewpoint mismatch.
\citet{zakka2021xirl} take a goal-driven perspective that seeks
to imitate task progress rather than match fine-grained
structural details to transfer between physical robots.
In contrast, our method does not rely on learning an
explicit cross-domain latent space between the agents,
nor does it rely on proxy tasks.
The Gromov-Wasserstein distance enables us to directly
compare the different spaces without a shared space.
The existing benchmark tasks we are aware
of assume access to a set of demonstrations from
\emph{both} agents whereas the experiments in our paper
\emph{only} assume access to expert demonstrations. Finally, other domain adaptation and transfer learning settings
use Gromov-Wasserstein variants, \eg for transfer
between word embedding spaces \citep{gwalign} and image spaces \citep{coot}.

\section{Preliminaries}
\textbf{Metric Markov Decision Process.} An infinite-horizon
discounted \emph{Markov decision Process (MDP)} is a tuple
$(S,A,R,P,p_0,\gamma)$ where $S$ and $A$ are state and action spaces,
$P:S \times A \rightarrow \Delta(S)$ is the transition function, $R: S
\times A \rightarrow \mathbb{R}$ is the reward function, $p_0 \in
\Delta(S)$ is the initial state distribution and $\gamma$ is the
discount factor. We equip MDPs with a distance $d: S
\times A \rightarrow \mathbb{R^+}$ and call the tuple
$(S,A,R,P,p_0,\gamma,d)$ a \emph{metric MDP}.

\textbf{Gromov-Wasserstein distance.}
Let $(\mathcal{X}, d_{\mathcal{X}}, \mu_{\mathcal{X}})$ and
$(\mathcal{Y}, d_{\mathcal{Y}}, \mu_{\mathcal{Y}})$ be two metric
measure spaces, where $d_{\mathcal{X}},d_{\mathcal{Y}}$ are
distances, and $\mu_{\mathcal{X}},\mu_{\mathcal{Y}}$ are
measures on their respective spaces\footnote{We use discrete spaces for readability but show empirical results in continuous spaces.}.
Optimal transport \citep{villani2009optimal,peyre2019computational}
studies how to compare measures.
We will use the \emph{Gromov-Wasserstein} distance \citep{memoli2011gromov}
between metric measure spaces, which has been theoretically generalized
and further studied in \citet{sturm2012space,peyre2016gromov,vayer2020contribution}
and is defined by 
\begin{align} \label{eq:defgw}
\mathcal{GW}((\mathcal{X}, d_{\mathcal{X}}, \mu_{\mathcal{X}}),(\mathcal{Y}, d_{\mathcal{Y}}, \mu_{\mathcal{Y}}))^2 = \min_{u \in \mathcal{U}(\mu_{\mathcal{X}},\mu_{\mathcal{Y}})} \sum_{\mathcal{X}^2 \times \mathcal{Y}^2} |d_{\mathcal{X}}(x,x') - d_{\mathcal{Y}}(y,y')|^2 u_{x,y}u_{x',y'},
\end{align}
where $\mathcal{U}(\mu_{\mathcal{X}},\mu_{\mathcal{Y}})$ is the set of couplings between the atoms of the measures defined by 
$$\mathcal{U}(\mu_{\mathcal{X}},\mu_{\mathcal{Y}}) = \left\{u \in \mathbb{R}^{\mathcal{X} \times \mathcal{Y}}\ \middle|\ \forall x \in \mathcal{X}, \sum_{y \in \mathcal{Y}} u_{x,y} = \mu_{\mathcal{X}}(x), \forall y \in \mathcal{Y}, \sum_{x \in \mathcal{X}} u_{x,y} = \mu_{\mathcal{Y}}(y)\right\}.$$

% of $\mu_{\mathcal{X}}$ and $\mu_{\mathcal{Y}}$ define.
$\mathcal{GW}$ compares the structure of two metric measure spaces
by comparing the pairwise distances within each space to find
the best isometry between the spaces. Figure \ref{fig:overview} illustrates this distance in the case of the metric measure spaces $(S_E \times A_E, d_E, \rho_{\pi_E})$ and $(S_A \times A_A, d_A, \rho_{\pi_A})$.

% In the special case where $\mu_{\mathcal{X}}$ and $\mu_{\mathcal{Y}}$ are uniform with finite support $\{x_1,...,x_T\}$ and $\{y_1,...,y_{T'}\}$ respectively, the Gromov-Wasserstein distance is
% \begin{align}
% \mathcal{GW}((\mathcal{X}, d_{\mathcal{X}}, \mu_{\mathcal{X}}),(\mathcal{Y}, d_{\mathcal{Y}}, \mu_{\mathcal{Y}}))^2 = \min_{\theta \in \Theta^{T \times T'}} \sum_{\substack{1 \leq i,i' \leq T \\ 1 \leq j,j' \leq T'}} |d_{\mathcal{X}}(x_i,x_{i'}) - d_{\mathcal{Y}}(y_j,y_{j'})|^2 \theta_{ij}\theta_{i'j'}
% \end{align}
% where $\Theta$ is the set of couplings between the atoms of the measures
% defined by
% %$\Theta^{T \times T'}$ is the set of all doubly-stochastic matrices of size $T \times T'$:
% % \bda{it seems strange to call these doubly-stochastic matrices since the columns/rows don't
% %   sum to 1, and they're not square}
% % \begin{align} \label{eq:deftheta}
% $$\Theta^{T \times T'} = \left\{\theta \in \mathbb{R}^{T \times T'}\ \middle|\ \forall i \in [T], \sum_{j \in [T']} \theta_{ij} = 1/T, \forall j \in [T'], \sum_{i \in [T]} \theta_{ij} = 1/T'\right\}.$$
% \end{align}

\section{Cross-Domain Imitation Learning via Optimal Transport }
% We use the Gromov-Wasserstein metric to compare policies from
% arbitrarily different MDPs for cross-domain imitation learning.
% \Cref{fig:overview} and \cref{alg:gwil} summarize the method.

\subsection{Comparing policies from arbitrarily different MDPs}
For a stationary policy $\pi$ acting on a metric MDP
$(S,A,R,P,\gamma,d)$, the \emph{occupancy measure} is:
$$ \rho_{\pi}: S \times A \rightarrow \mathbb{R} \qquad
\rho(s,a) = \pi(a|s)\sum_{t=0}^{\infty} \gamma^t P(s_t=s|\pi).$$
We compare
policies from arbitrarily different MDPs in terms of their occupancy
measures.

% We will compare policies in terms of their occupancy measure defined below. two policies $\pi$ and $\pi'$

% \begin{definition}[Occupancy measure]\label{def:1}
% Given two distances $d_S$ and $d_{S'}$ defined on $S$ and $S'$
% respectively, we define the distance between $\tau$ and $\tau'$ as
% the Gromov-Wasserstein distance between the metric measure spaces
% $(S, d_S, \mu_{\tau})$ and $(S', d_{S'}, \mu_{\tau'})$:
% \begin{equation}\label{eq:1}
% \begin{aligned}
% \mathcal{GW}(\tau, \tau') & =  \mathcal{GW}(((S, d_S, \mu_{\tau}),(S', d_{S'}, \mu_{\tau'})) \\ & =  \min_{\theta \in \Theta^{T \times T'}} \sum_{\substack{1 \leq i,i' \leq T \\ 1 \leq j,j' \leq T'}} |d_S(s_i,s_{i'}) - d_{S'}(s'_j,s'_{j'})|^2 \theta_{ij}\theta_{i'j'}
% \end{aligned}
% \end{equation}
% \end{definition}

% trajectories $\tau = (s_1,a_1...,a_{T}) \in (S\times A)^T $ and $\tau =
% (s'_1,...,s'_{T'}) \in S'^{T'}$ sampled from $M = (S,A,R,T)$ and $M' =
% (S',A',R',T')$ respectively. We map $\tau$ (resp. $\tau'$) to the
% uniform probability measure $\mu_{\tau}: S \in [0,1]$
% (resp. $\mu_{\tau'}: S' \in [0,1]$) defined by $s \mapsto
% \frac{\mathbb{1}(s\in \tau)}{T}$ (resp.
% $s \mapsto \frac{\mathbb{1}(s\in \tau')}{T}$).
% Notice that this is a well-defined probability distribution because we
% can always add the time step to the state to distinguish between two
% potentially identical states in the trajectory. In the remaining we denote by $\bar{\tau}$ the set $\{s_1,...,s_T\}$.

% [maybe put invariant wasserstein as appendix, shorter]

\begin{definition}[Gromov-Wasserstein distance between (isomorphic classes of) policies\footnote{We later show that it is actually not a distance on policies but on isomorphic classes of policies.}]\label{def:1}
  Given an expert policy $\pi_E$ and an agent policy $\pi_A$ acting, respectively, on
  $$M_E = (S_E,A_E,R_E,P_E,T_E,d_E) \quad {\rm and} \quad M_A=(S_A,A_A,R_A,P_A,T_A,d_A).$$
  We define the Gromov-Wasserstein distance between $\pi_E$ and $\pi_A$ as the Gromov-Wasserstein distance between the metric measure spaces $(S_E \times A_E, d_E, \rho_{\pi_E})$ and $(S_A\times A_A, d_A, \rho_{\pi_A})$\footnote{We always consider a policy in the context of the underlying metric MDP, such that every policy acting on $(S,A,R,P,T,d_E)$ are different from every policy acting on $(S,A,R,P,T,d_A)$ as soon as $d_E\neq d_A$. This guarantees that the Gromov-Wasserstein distance respects the identity of indiscernibles.}:
\begin{equation}\label{eq:1}
\begin{aligned}
\mathcal{GW}(\pi, \pi') & =  \mathcal{GW}((S_E \times A_E, d_E, \rho_{\pi_E}),(S_A\times A_A, d_A, \rho_{\pi_A})) .
% & =  \min_{u \in \mathcal{U}(\rho_{\pi},\rho_{\pi'})} \int_{(S_E \times A_E)^2 \times (S_A \times A_A)^2} |d_E((s_E,a_E),(s_E',a_E')) - d_A((s_A,a_A),(s_A',a_A'))|^2 du((s_E,a_E),(s_A,a_A))du((s_E',a_E'),(s_A',a_A'))
% \min_{\theta \in \Theta^{T \times T'}} \sum_{\substack{1 \leq i,i' \leq T \\ 1 \leq j,j' \leq T'}} |d_S(s_i,s_{i'}) - d_{S'}(s'_j,s'_{j'})|^2 \theta_{ij}\theta_{i'j'}
\end{aligned}
\end{equation}
\end{definition}

We now define isomorphisms between policies by comparing
the state-action marginals and show that $\mathcal{GW}$ defines a distance
between them.
\Cref{fig:isomorphic-policies}
illustrates simple isomorphic policies.

\begin{definition}[Isomorphic policies]\label{def:2}
  Two policies $\pi_E$ and $\pi_A$ are isomorphic if there exists
  a bijection $\phi: \supp[\rho_{\pi_E}] \rightarrow \supp[\rho_{\pi_A}]$ that
  satisfies for all $(s_E,a_E),({s_E}',{a_E}') \in \supp[\rho_{\pi_E}]$ and $(s_A,a_A) \in \supp[\rho_{\pi_A}]$:
 \begin{align}
     & d_E\left((s_E,a_E),({s_E}',{a_E}')\right) = d_A\left(\phi(s_E,a_E),\phi({s_E}',{a_E}')\right) \\
     & \rho_{\pi_A}(s_A,a_A) = \rho_{\pi_E}(\phi^{-1}(s_E,a_E))
 \end{align}
% for all $(s_E,a_E),({s_E}',{a_E}') \in \supp[\rho_{\pi_E}]^2$:
% $$d_E\left((s_E,a_E),({s_E}',{a_E}')\right) = d_A\left(\phi(s_E,a_E),\phi({s_E}',{a_E}')\right)$$
In other words, $\phi$ is an isometry between $(\supp[\rho_{\pi_E}],d_E)$ and $(\supp[\rho_{\pi_A}],d_A)$ and $\rho_{\pi_A}$ is the push-forward measure $\phi_{\sharp}(\rho_{\pi_E})$.
% $\tau = (s_1,...,s_T) \in (S,d_S)^T$ and $\tau' = (s'_1,...,s'_{T'}) \in (S',d_{S'})^T$ are isometric if there is a bijective map $\phi: \{s_1,...,s_T\} \rightarrow \{s'_1,...,s'_{T'}\}$ such that for any $i,j \in [T]$, $d_S(s_i, s_j) = d_{S'}(\phi(s_i), \phi(s_j))$. In other words, $\phi$ is an isometry from $(\bar{\tau},d_S)$ to $(\bar{\tau}',d_{S'})$. 
\end{definition}

\begin{proposition}\label{prop:1}
$\mathcal{GW}$ defines a metric on the collection of all isomorphic classes of policies.
\end{proposition}
\begin{proof}
By \cref{def:1}, $\mathcal{GW}(\pi_E, \pi_A)=0$ if and only if $\mathcal{GW}((S_E, d_E, \rho_{\pi_E}),(S_A, d_A, \rho_{\pi_A})) = 0$. By \citet[Theorem 5.1]{memoli2011gromov}, this is true if and only if there is an isometry $\phi: \supp[\rho_{\phi_E}] \rightarrow \supp[\rho_{\phi_A}]$ such that $\rho_{\pi_A} = \phi_{\sharp}(\rho_{\pi_E})$. By \cref{def:2}, this is true if and only if $\pi_A$ and $\pi_E$ are isomorphic. The symmetry and triangle inequality follow from \citet[Theorem 5.1]{memoli2011gromov}.
\end{proof}

% Similarly, we can define the invariant Wasserstein between the trajectories, which will automatically incorporate invariances to rotations and translations, and which applies if trajectories are living on metric spaces of different dimensions:

% \begin{align}\label{eq:iw_states}
% \mathcal{IW}(\tau, \tau')= \min_{\theta \in \Theta^{T \times T'}} \min_{f\in \mathcal{F}} \sum_{\substack{1 \leq i\leq T \\ 1 \leq j \leq T'}} ||s_i - f(s_{j})||^2\theta_{ij}.
% \end{align}

The next theorem\footnote{Our proof is in finite state-action spaces for readability and can be directly extended to infinite spaces.} gives a sufficient condition to recover, by minimizing $\mathcal{GW}$, an optimal policy\footnote{A policy is optimal in the MDP $(S,A,R,P,\gamma,d)$ if it maximizes the expected return $\mathbb{E}\sum_{t=0}^{\infty}R(s_t,a_t)$.} in the agent's domain up to an isometry.  

\begin{theorem}\label{the:1}
  Consider two MDPs
  $$M_E=(S_E,A_E,R_E,P_E,p_E,\gamma) \quad {\rm and} \quad
  M_A=(S_A,A_A,R_A,P_A,p_A,\gamma).$$
  Suppose that there exists four
distances $d_E^S, d_E^A, d_A^S, d_A^A$ defined on $S_E$, $A_E$, $S_A$
and $A_E$ respectively, and two isometries
$\phi: (S_E,d_E^S) \rightarrow (S_A,d_A^S)$ and $\psi: (A_E,d_E^S) \rightarrow
(A_S,d_A^S)$ such that for all $(s_E,a_E,s_E') \in S_E \times A_E
\times S_E$ the three following conditions hold:
\begin{align}
     R(s_E,a_E) = R_A(\phi(s_E),\psi(a_E))  \label{cdt:1}\\
    {P_E}_{s_E,a_E}(s_E') = {P_A}_{\phi(s_E)\psi(a_E)}(\phi(s_E')) \label{cdt:2} \\
    p_E(s_E) = p_A(\phi(s_E)) \label{cdt:3}.
\end{align}
Consider an optimal policy $\pi^*_E$ in $M_E$. Suppose that $\pi_{GW}$ minimizes $\mathcal{GW}(\pi^*_E,\pi_{GW})$ with
$$d_E:(s_E,a_E)\mapsto d_E^S(s_E) + d_E^A(a_E) \quad {\rm and} \quad d_A:(s_A,a_A)\mapsto d_A^S(s_A)+ d^A_A(a_A).$$
Then $\pi_{GW}$ is isomorphic to an optimal policy in $M_A$.
\end{theorem}

\begin{proof}
Consider the occupancy measure $\rho^*_A: S_A \times A_A \rightarrow
\mathbb{R}$ given by
$$(s_A,a_A) \mapsto \rho_{\pi^*_E}(\phi^{-1}(s_A), \psi^{-1}(a_A)).$$
We first show that $\rho^*_A$ is feasible in $M_A$, \ie there exists
a policy $\pi^*_A$ acting in $M_A$ with occupancy measure $\rho^*_A$
(a). Then we show that $\pi^*_A$ is optimal in $M_A$ (b) and is
isomorphic to $\pi^*_E$ (c). Finally we show that $\pi_{GW}$ is
isomorphic to $\pi^*_A$, which concludes the proof (d).

(a) Consider $s_A \in S_A$. By definition of $\rho^*_A$,
$$\sum_{a_A \in A_A}\rho^*_A(s_A) = \sum_{a_A \in A_A}\rho_{\pi^*_E}(\phi^{-1}(s_A), \psi^{-1}(a_A)) = \sum_{a_E \in A_E}\rho_{\pi^*_E}(\phi^{-1}(s_A), a_E).$$
Since $\rho_{\pi^*_E}$ is feasible in $M$, it follows from
\citet[Theorem 6.9.1]{puterman2014markov} that
$$\sum_{a_E \in A_E}\rho_{\pi^*_E}(\phi^{-1}(s_A), a_E) = p_E(\phi^{-1}(s_A)) + \gamma
\sum_{s_E \in S_E,a_E \in A_E} {P_E}_{s_E,a_E}(\phi^{-1}(s_A))+
\rho_{\pi^*_E}(s_E, a_E).$$
By conditions \ref{cdt:2} and \ref{cdt:3}
and by definition of $\rho^*_A$,
\begin{equation*}
\begin{aligned}
&p_E(\phi^{-1}(s_A)) + \gamma
  \sum_{s_E \in S_E,a_E \in A_E} {P_E}_{s_E,a_E}(\phi^{-1}(s_A))+
  \rho_{\pi^*_E}(s_E, a_E) \\
&\hspace{3mm}=  p_A(s_A) + \gamma \sum_{s_E \in S_E,a_E \in A_E}
  {P_A}_{\phi(s_E),\psi(a_E)}(s_A)+ \rho^*_A(\phi(s_E), \psi(a_E)) \\
&\hspace{3mm}= p_A(s_A) + \gamma \sum_{s_A' \in S_A,a_A \in A_A} {P_A}_{s_A',a_A}(s_A)+
  \rho^*_A(s_A', a_A).
\end{aligned}
\end{equation*}
It follows that
$$\sum_{a_A \in A_A}\rho^*_A(s_A)
= p_A(s_A) + \gamma \sum_{s_A' \in S_A,a_A \in A_A} {P_A}_{s_A',a_A}(s_A)+
\rho^*_A(s_A', a_A).$$
Therefore, by
\citet[Theorem 6.9.1]{puterman2014markov},
$\rho^*_A$ is feasible in $M_A$, \ie there exists a
policy $\pi^*_A$ acting in $M_A$ with occupancy measure $\rho^*_A$.

(b) By condition \ref{cdt:3} and definition of $\rho^*_A$, the
expected return of $\pi^*_A$ in $M_A$ is then
\begin{equation*}
\begin{aligned}
&\sum_{s_A \in S_A, a_A \in A_A} \rho^*_A(s_A,a_A)R_A(s_A,a_A) \\
&\hspace{10mm}= \sum_{s_A \in S_A, a_A \in A_A}
\rho^*_E(\phi^{-1}(s_A),\psi^{-1}(a_A))R_E(\phi^{-1}(s_A),\psi^{-1}(a_A)) \\
&\hspace{10mm}= \sum_{s_E \in S_E, a_E \in A_E}
\rho^*_E(s_E,a_E)R_E(s_E,a_E)
\end{aligned}
\end{equation*}
Consider any policy $\pi_A$ in
$M'$. By condition \ref{cdt:3}, the expected return of $\pi_A$ is
$$\sum_{s_A \in S_A, a_A \in A_A} \rho_{\pi_A}(s_A,a_A)R_A(s_A,a_A) =
\sum_{s_E \in S_E, a_E \in A_E}
\rho_{\pi_A}(\phi(s_E),\psi(a_E))R_E(s_E,a_E).$$
Using the same arguments that we used to show that
$\rho^*_A$ is feasible in $M'$, we
can show that
$$(s_E,a_E) \mapsto \rho_{\pi_A}(\phi(s_E),\psi(a_E))$$
is feasible in $M$.
It follows by optimality of $\pi^*_E$ in $M$ that
\begin{equation*}
\begin{aligned}
\sum_{s_E \in S_E, a_E \in A_E}
\rho_{\pi_A}(\phi(s_E),\psi(a_E))R_E(s_E,a_E)
&\leq \sum_{s_E \in S_E,
  a_E \in A_E} \rho_{\pi^*_E}(\phi(s_E),\psi(a_E))R_E(s_E,a_E) \\
&= \sum_{s_A \in S_A, a_A \in A_A} \rho^*_A(s_A,a_A)R_A(s_A,a_A).
\end{aligned}
\end{equation*}
It follows that $\pi^*_A$ is optimal in $M'$.

(c) Notice that $$\xi: (s_E,a_E) \mapsto (\phi(s_E), \psi(a_E))$$ is an
isometry between $(S_E \times A_E,d_E)$ and $(S_A \times A_A,d_A)$,
where $d_E$ and $d_A$ and given, resp., by
$$(s_E,a_E) \mapsto d^S_E(s_E)+ d^A_E(a_E) \quad {\rm and}\quad
(s_A,a_A) \mapsto d^S_A(s_A)+ d^A_A(a_A).$$
Furthermore, by definition, $\rho^{\star}_A = \xi_{\sharp}(\rho^*_E)$.
Therefore by definition \ref{def:2}, $\pi^*_A$ is
isomorphic to $\pi^*_E$.

(d) Recall from the statement of the theorem that $\pi_{GW}$ is a
minimizer of $\mathcal{GW}(\pi^*_E, \pi_{GW})$. Since $\pi^*_A$ is
isomorphic to $\pi^*_E$, it follows from \cref{prop:1} that
$\mathcal{GW}(\pi^*_E, \pi^*_A) = 0$. Therefore $\mathcal{GW}(\pi^*_E,
\pi_{GW})$ must be 0. By \cref{prop:1}, it follows that
there exists an isometry
$$\chi:(\supp[\rho^*_E], d_E) \rightarrow
(\supp[\rho_{\pi_{GW}}], d_A)$$
such that $\rho_{\pi_{GW}} = \chi_{\sharp}(\rho^*_E)$. Notice that $\chi \circ \xi^{-1}|_{\supp[\rho^*_A]}$
is an isometry from $(\supp[\rho^*_A],
d_A)$ to $(\supp[\rho_{\pi_{GW}}], d_A)$ and $\rho_{\pi_{GW}} = (\chi \circ \xi^{-1}|_{\supp[\rho^*_A]})_{\sharp}(\rho^*_A)$.
It follows by definition $\ref{def:2}$ that $\pi_{GW}$ is isomorphic to $\pi^*_A$, an
optimal policy in $M_A$, which concludes the proof.
\end{proof}

\begin{remark}
\Cref{the:1} shows the possibilities and limitations of our method. It
shows that our method can recover optimal policies even though
arbitrary isometries are applied to the state and action spaces of the
expert's domain. Importantly, we don't need to know the isometries,
hence our method is applicable to a wide range of settings. We will
show empirically that our method produces strong results in other
settings where the environment are not isometric and don't even have
the same dimension. However, a limitation of our method is that
it recovers optimal policy only up to isometries. We will see that in
practice, running our method on different seeds enables to find an
optimal policy in the agent's domain.
\end{remark}

\begin{algorithm}[t]
  \begin{algorithmic}
    \State \textbf{Inputs:} expert demonstration $\tau$, metrics on the expert ($d_E$) and agent ($d_A$) space
    \State Initialize the imitation agent's policy $\pi_\theta$ and value estimates $V_\theta$
    \While{Unconverged}
    \State Collect an episode $\tau'$
    \State Compute $\mathcal{GW}(\tau, \tau')$ 
    \State Set pseudo-rewards $r$ with \cref{eq:rew}
    \State Update $\pi_\theta$ and $V_\theta$ to optimize the pseudo-rewards
    \EndWhile
  \end{algorithmic}
  \caption{Gromov-Wasserstein imitation learning
  from a single expert demonstration.}
  \label{alg:gwil}
\end{algorithm}

% \subsection{Computing $\mathcal{GW}$ in practice with finite samples}
% We approximate the occupancy measures of $\pi$ by $\hat{\rho}(s,a) =
% \frac{1}{T} \sum_{t=1}^{T} \mathbb{1}(s=s_t \wedge a=a_t)$ where $\tau
% = (s_1,a_1,..,s_T,a_T)$ is a finite trajectory collected with
% $\pi$. Assuming that all state-action pairs in the trajectory are
% different\footnote{We can add the time step to the state to
% distinguish between two identical state-action pairs in the
% trajectory.}, $\hat{\rho}$ is a uniform distribution. Given an expert
% trajectory $\tau_E$ and an agent trajectory $\tau_A$, the approximate
% Gromov-Wasserstein distance between them is \sam{this is the exact gromov-wasserstein between empirical estimates of the occupancy measures, so I think we should clarify what we mean by approximate}

% \begin{align}
%   \small
%   \widehat{\mathcal{GW}}(\tau_E, \tau_A) = \min_{\theta \in \Theta^{T_E \times T_A}} \sum_{\substack{1 \leq i,i' \leq T_E \\ 1 \leq j,j' \leq T_A}} |d_E((s^E_i,a^E_i),(s^E_{i'},a^E_{i'})) - d_A((s^A_j,s^A_j),(s^A_{j'},a^A_{j'}))|^2 \theta_{ij}\theta_{i'j'}.
% \end{align}
% where $\Theta$ is the set of all double-stochastic matrices defined in \cref{eq:deftheta}.

% \subsection{Using $\mathcal{GW}$ and $\mathcal{IW}$ for Cross-Domain Imitation Learning}
\subsection{Gromov-Wasserstein Imitation Learning}
Minimizing $\mathcal{GW}$ between an expert and agent requires derivatives
through the transition dynamics, which we typically don't have access to.
We introduce a reward proxy suitable for training an agent's policy that
minimizes $\mathcal{GW}$ via RL.
\Cref{fig:overview} illustrates the method.
For readability, we combine expert state and action variables $(s_E, a_E)$ into
single variables $z_E$, and similarly for agent state-action pairs.
Also, we define $Z_E = S_E\times A_E$ and $Z_A = S_A\times A_A$.

\begin{definition}
\label{def:rgw}
  Given an expert policy $\pi_E$ and an agent policy $\pi_A$, the Gromov-Wasserstein reward of the agent is defined as $r_{\mathcal{GW}}:\supp[\rho_{\pi_A}] \rightarrow \mathbb{R}$ given by
  \begin{equation*}
    \small
    r_{\mathcal{GW}}(z_A) =  -\frac{1}{\rho_{\pi}(z_A)}\sum_{\substack{z_E \in Z_E \\ z'_E \in Z_E\\z_A' \in Z_A}} |d_E(z_E,z_E')) - d_A(z_A,z_A')|^2 u^\star_{z_E,z_A}u^\star_{z_E',z_A'}
  \end{equation*}
%   \begin{equation*}
%     \small
%     r_{\mathcal{GW}}(s_A,a_A) = - \int_{\substack{(s_E,a_E),(s'_E,a'_E) \in (S_E\times A_E)^2 \\ z'_A \in S_A \times A_A}} |d_E((s_E,a_E),(s'_E,a'_E)) - d_A((s_A,a_A),z'_A)|^2 du^*((s_A,a_A),(s_E,a_E))du^*z'_A,(s'_E,a'_E))
%   \end{equation*}
%   \begin{align}
% r_{\mathcal{GW}}(s_A,a_A) = \int_{\substack{(s_E,a_E),(s'_E,a'_E) \in (S_E\times A_E)^2 \\ z'_A \in S_A \times A_A}} |d_E((s_E,a_E),(s'_E,a'_E)) - d_A((s_A,a_A),z'_A)|^2 u^*((s_A,a_A),(s_E,a_E))u^*z'_A,(s'_E,a'_E)) 
% \end{align}
where $u^\star$ is the coupling minimizing objective \ref{eq:defgw}.
\end{definition}

\begin{proposition}\label{prop:2}
  If $\pi_A$ minimizes $\mathcal{GW}(\pi_E, \pi_A)$, then $\pi_A$ is an optimal policy for the reward $r_{\mathcal{GW}}$ as defined in definition \ref{def:rgw}.
\end{proposition}
\begin{proof}
  Suppose that $\pi_A$ minimizes $\mathcal{GW}(\pi_E, \pi_A)$, then by definition \ref{def:1} $\pi_A$ maximizes
  \begin{equation*}
    \small
    \begin{aligned}
    = & -\sum_{\substack{z_E \in Z_E \\ z'_E \in Z_E\\z_A \in Z_A\\z'_A \in Z_A}} |d_E(z_E,z'_E) - d_A(z_A,z'_A)|^2 u^\star_{z_A,z_E}u^\star_{z'_A,z'_E} \\
    %   \mathop{\mathbb{E}}_{z_A\sim \rho_{\pi_A}}r_{\mathcal{GW}}(z_A)
    %   \left[\sum_{1\leq t\leq T_A} r(s_A_t, a_t^A)\right] 
    = & -\sum_{z_A \in \supp[\rho_{\pi_A}]} \frac{\rho_{\pi_A}(z_A)}{\rho_{\pi_A}(z_A)}\sum_{\substack{z_E \in Z_E \\ z'_E \in Z_E\\z'_A \in Z_A}} |d_E(z_E,z'_E) - d_A(z_A,z'_A)|^2 u^\star_{z_A,z_E}u^\star_{z'_A,z'_E} \\
    %   = & -\sum_{\substack{z_E \in Z_E \\ z'_E \in Z_E\\z_A \in Z_A\\z'_A \in Z_A}} |d_E(z_E,z'_E) - d_A(z_A,z'_A)|^2 u^\star_{z_A,z_E}u^\star_{z'_A,z'_E} \\
    \end{aligned}
    % \vspace{-7mm}
  \end{equation*}
Therefore, by \citet[Theorem 6.9.4]{puterman2014markov}, $\pi_A$ is an optimal policy for reward $r_{\mathcal{GW}}$.
\end{proof}

In practice we approximate the occupancy measures of $\pi$ by $\hat{\rho}_{\pi}(s,a) =
\frac{1}{T} \sum_{t=1}^{T} \mathbb{1}(s=s_t \wedge a=a_t)$ where $\tau
= (s_1,a_1,..,s_T,a_T)$ is a finite trajectory collected with
$\pi$. Assuming that all state-action pairs in the trajectory are
different\footnote{We can add the time step to the state to
distinguish between two identical state-action pairs in the
trajectory.}, $\hat{\rho}$ is a uniform distribution. Given an expert
trajectory $\tau_E$ and an agent trajectory $\tau_A$ \footnote{Note that the Gromov-Wasserstein distance defined in equ. (6) does not depend on the temporal ordering of the trajectories.}, the
(squared) Gromov-Wasserstein distance between the empirical occupancy measures is 
% \sam{this is the exact gromov-wasserstein between empirical estimates of the occupancy measures, so I think we should clarify what we mean by approximate}

\begin{align}\label{eq:defgwuniform}
  \small
  \mathcal{GW}^2(\tau_E,\tau_A) = \min_{\theta \in \Theta^{T_E \times T_A}} \sum_{\substack{1 \leq i,i' \leq T_E \\ 1 \leq j,j' \leq T_A}} |d_E((s^E_i,a^E_i),(s^E_{i'},a^E_{i'})) - d_A((s^A_j,s^A_j),(s^A_{j'},a^A_{j'}))|^2 \theta_{i,j}\theta_{i',j'}
\end{align}
where $\Theta$ is the set of is the set of couplings between the atoms of the uniform measures defined by
$$\Theta^{T \times T'} = \left\{\theta \in \mathbb{R}^{T \times T'}\ \middle|\ \forall i \in [T], \sum_{j \in [T']} \theta_{i,j} = 1/T, \forall j \in [T'], \sum_{i \in [T]} \theta_{i,j} = 1/T'\right\}.$$

% \bda{Hmm, is there a way we could write this (and the alg) to be between stationary distributions
%   and then say at some point that we just instantiate it with just a single trajectory?
%   Feels incomplete if we just show the algorithm and this section for when we just have
%   a single trajectory.}\sam{agree, we could also mention the multiple trajectory case - which can be done by either concatenation or summing gws between agent and each expert}
In this case the reward is given for every state-action pairs in the trajectory by:
\begin{equation}
  \begin{aligned}
    r(s^A_j, s^A_j) = -T_A\sum_{\substack{1 \leq i,i' \leq T_E \\ 1 \leq j' \leq T_A}} |d_E((s^E_i,a^E_i),(s^E_{i'}, a^E_{i'})) - d_A((s^A_j,s^A_j),(s^A_{j'},a^A_{j'}))|^2 \theta^\star_{i,j}\theta^\star_{i',j'}
    \label{eq:rew}
  \end{aligned}
\end{equation}
where $\theta^\star$ is the coupling minimizing objective \ref{eq:defgwuniform}.

In practice we drop the factor $T_A$ because it is the same for every state-action pairs in the trajectory.

% Consider an expert policy $\pi_E$ generating a demonstration $\tau_E = (s_E_1,a_E_1,...,s_E_{T_E},a_E_{T_E})$
% and an
% agent policy $\pi_A$ interacting with an environment and generating a trajectory $\tau_A = (s_A_1,a_A_1,...,s_A_{T_A},a_A_{T_A})$.
% Ideally we would like to find the $\widehat{\mathcal{GW}}$-minimal policy to
% the demonstration, \ie optimizing
% \begin{equation}
%   \min_{\pi_A} \mathbb{E}_{\tau_A\sim\pi_A}\left[\widehat{\mathcal{GW}}(\tau_E,\tau_A)\right]
%   % , \quad \mathrm{or} \quad \min_{\pi_A} \mathbb{E}_{\tau_A\sim\pi_A}[\hat{\mathcal{IW}}(\tau_E,\tau_A)] 
%   \label{eq:gw_obj}
% \end{equation}
\begin{remark}
The construction of our reward proxy is defined for any occupancy measure and extends to previous work optimizing optimal transport quantities via RL that assumes uniform occupancy measure in the form of a trajectory to bypass the need for derivatives through the transition dynamics \citep{dadashi2020primal,papagiannis2020imitation}.
\end{remark}

\textbf{Computing the pseudo-rewards.}
We compute the Gromov-Wasserstein distance using \citet[Proposition 1]{peyre2016gromov} and its gradient using \citet[Proposition 2]{peyre2016gromov}. To compute the coupling minimizing \ref{eq:defgwuniform}, we use the conditional gradient method as in \cite{ferradans2013regularized}.

\textbf{Optimizing the pseudo-rewards.}
The pseudo-rewards we obtain from $\mathcal{GW}$ for the imitation agent
enable us to turn the imitation learning problem into a
reinforcement learning problem \citep{sutton2018reinforcement}
to find the optimal policy for
the Markov decision process induced by the pseudo-rewards.
We consider agents with continuous state-action spaces and
thus do policy optimization with the soft actor-critic algorithm
\citep{haarnoja2018soft}.
\Cref{alg:gwil} sums up GWIL in the case where a single expert trajectory is given to approximate the expert occupancy measure.

\section{Experiments}
% [say that the sparse reward task is hard exploration for SAC]

% \bda{I THINK Should we emphasize again why we don't compare
%   to the existing works here since they use demonstrations
%   from both agents?
%   Or would that seem too defensive?
% }

We propose a benchmark set for cross-domain IL methods consisting of 3 tasks and
aiming at answering the following questions:

\begin{enumerate}
\item \emph{Does GWIL recover optimal behaviors when the agent
    domain is a rigid transformation of the expert domain?}
  Yes, we demonstrate this with the maze in \cref{sec:exp:maze}.
\item \emph{Can GWIL recover optimal behaviors when the agent
    has different state and action spaces than the expert?}
  Yes, we show in \cref{sec:exp:cartpole} for \emph{slightly} different
  state-action spaces between the cartpole and pendulum,
  and in \cref{sec:exp:cheetah} for \emph{significantly} different
  spaces between a walker and cheetah.
\end{enumerate}

To answer these three questions, we use simulated continuous control tasks
implemented in Mujoco \citep{todorov2012mujoco}
and the DeepMind control suite \citep{tassa2018deepmind}. We include videos of learned policies on our project site\footnote{\href{https://arnaudfickinger.github.io/gwil/}{https://arnaudfickinger.github.io/gwil/}}. In all settings we use the Euclidean metric within the expert and agent spaces for $d_E$ and $d_A$. 
% \bda{In all settings we use the Euclidean metric within the expert
% and agent spaces for $d_E$ and $d_A$.}

\subsection{Agent domain is a rigid transformation of the expert domain}
% [maybe train without reward and show, we recover the optimal behavior, take a look at video ion project website]
% [to show the utility on fownream, we add a psparse reqrd, and it show that it can select the reight behavior although the reward is uniformative for state of the art RL algo]
% [we recover intersting behvor, should converge to GW, exploration in sparse reward envt, we can transfer structure of trajectories, strip for cartpole (without reward), we show that it can enable easier exploration in new domain with example sparse maze and sparse cartpole, compared to]
\label{sec:exp:maze}
We evaluate the capacity of IL methods to transfer to rigid
transformation of the expert domain by using the PointMass Maze
environment from \citet{hejna2020hierarchically}. The agent's domain
is obtained by applying a reflection to the expert's maze.
% , as seen on \cref{fig:maze_illustration}.
% The agent's domain is equipped with a
% sparse reward that does not provide enough information to enable SAC
% to learn the optimal behavior.
% \bda{We don't need these since the filmstrip shows the agents}
% \begin{figure}[t]
%   \centering
%   \includegraphics[width=0.6\textwidth]{fig/maze_illustration.jpg}
%   \caption{
%   The first benchmark task consists of a maze and its reflection. Left is the expert domain and right is the agent domain.
%   \bda{Polish: Can we easily show the top view of the mazes with the path of the
%   expert and agent trajectories overlaid on top of them?
% }
% }
%   \label{fig:maze_illustration}
% \end{figure}
This task satisfies the condition of \cref{the:1} with
$\phi$ being the reflection through the central horizontal plan and
$\psi$ being the reflection through the $x$-axis in the action
space. Therefore by \cref{the:1}, the agent's optimal policy should be
isomorphic to the policy trained using GWIL.
By looking at the geometry
of the maze, it is clear that every policy in the isometry class of an
optimal policy is optimal. Therefore we expect GWIL to recover an
optimal policy in the agent's domain. \Cref{fig:maze_result} shows
that GWIL indeed recovers an optimal policy.

\begin{figure}[t]
\centering
\begin{subfigure}{.30\textwidth}
  \centering
  \includegraphics[width=\linewidth]{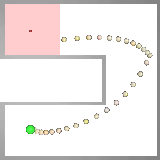}
  \caption{Expert}
  \label{fig:minienv} 
\end{subfigure}\hfil%
\begin{subfigure}{.30\textwidth}
  \centering
  \includegraphics[width=\linewidth]{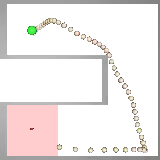}
  \caption{Agent}
\end{subfigure}
\caption{Given a single expert trajectory in the expert's domain (a), GWIL recovers an optimal policy in the agent's domain (b) without any external reward, as predicted by \cref{the:1}. The green dot represents the initial state position and the episode ends when the agent reaches the goal represented by the red square.}
% \vspace{-0.4cm}
\label{fig:maze_result}
\end{figure}

% \begin{figure}[t]
%   \centering
%   \includegraphics[width=1\textwidth]{fig/mz.jpg}
%   \caption{
%     Given a single expert trajectory in the expert's domain (top), GWIL recovers the optimal policy in the agent's domain (bottom) without any external reward, as predicted by \cref{the:1}. 
%     % \bda{RESOLVED nned to add agentThis one's especially hard to see what's going on}
%   }
%   \label{fig:maze_result}
% \end{figure}

\begin{figure}[t]
  \centering
  \includegraphics[width=1\textwidth]{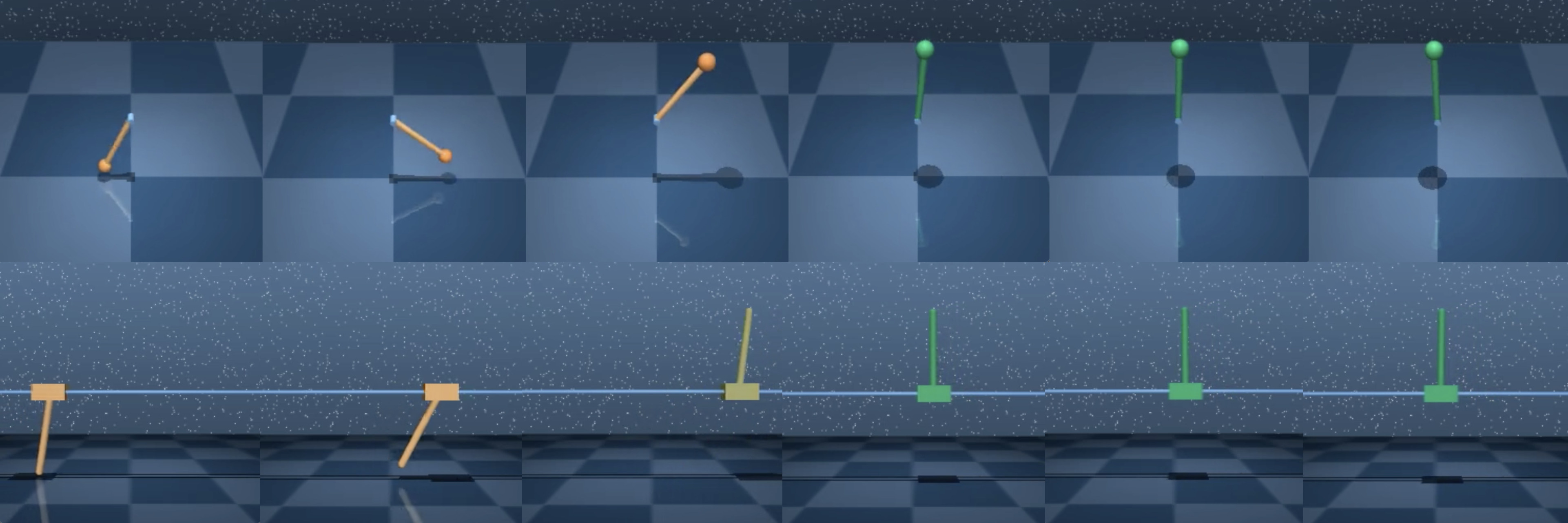}
  \caption{
    Given a single expert trajectory in the pendulum's domain (above), GWIL recovers the optimal behavior in the agent's domain (cartpole, below) without any external reward.
    % \bda{Polish: Could add white background and squeeze another frame or 2 into each line.}
  }
  \label{fig:cartpole_result}
\end{figure}

\subsection{Agent and the expert have slightly different state and action spaces}
\label{sec:exp:cartpole}
% [maybe try invariant wassserstein on it with rotation and show that IW perform worse but GW perform well because we dont need exact invariance]
% change the pic to https://wandb.ai/arnaudfickinger/gw_il/runs/23exmt73/overview?workspace=user-arnaudfickinger
% [maybe try without actions, and with selected state dimensions]
% [we optimize GW and pick the best one with sparse reward]
We evaluate here the capacity of IL methods to transfer to
transformation that does not have to be rigid but description map
should still be apparent by looking at the domains. A good example of
such transformation is the one between the pendulum and cartpole.
The pendulum is our expert's domain while
cartpole constitutes our agent's domain.
% as shown in \cref{fig:cartpole_illustration}.
The expert is trained on the swingup task.
% The agent's domain is equipped with a sparse reward that is not informative enough to be used by SAC. Specifically, the agent receives a positive reward if it can maintain the unstable position during the 10 previous steps, and 0 otherwise.
%
% \bda{We don't need these since the filmstrip shows the agents}
% \begin{figure}[t]
%   \centering
%   \includegraphics[width=0.6\textwidth]{fig/cartpole_illustration.jpg}
%   \caption{
%     The second benchmark task consists of the Pendulum and the Cartpole domains. Left is the expert domain and right is the agent domain.
%   }
%   \label{fig:cartpole_illustration}
% \end{figure}
%
Even though the transformation is not rigid, GWIL is able to recover
the optimal behavior in the agent's domain as shown in
\cref{fig:cartpole_result}. Notice that pendulum and cartpole do not
have the same state-action space dimension: The pendulum has 3 dimensions
while the cartpole has 5 dimensions. Therefore GWIL can indeed be applied
to transfer between problems with different dimension.

% \begin{figure}[t]
%   \centering
%   \includegraphics[width=0.6\textwidth]{fig/cartpole_result.png}
%   \caption{
%     GWIL recovers the agent's optimal behavior. The sparse reward does not provide enough information to enable SAC to learn the optimal behavior.  
%   }
%   \label{fig:cartpole_result}
% \end{figure}

\subsection{Agent and the expert have significantly different state and action spaces}
\label{sec:exp:cheetah}
% [maybe think about initial cdt]
We evaluate here the capacity of IL methods to transfer to non-trivial
transformation between domains. A good example of such transformation
is two arbitrarily different morphologies from the DeepMind Control
Suite such as the cheetah and walker. The cheetah constitutes our expert's
domain while the walker constitutes our agent's domain.
% as shown in \cref{fig:cartpole_illustration}.
The expert is trained on the run task.

% \bda{We don't need these since the filmstrip shows the agents}
% \begin{figure}[t]
%   \centering
%   \includegraphics[width=0.6\textwidth]{fig/walker_illustration.jpg}
%   \caption{
%     The third benchmark task consists of the cheetah and the walker domains. Left is the expert domain and right is the agent domain.
%   }
%   \label{fig:cartpole_illustration}
% \end{figure}

Although the mapping between these two domains is not trivial,
minimizing the Gromov-Wasserstein solely enables the walker to
interestingly learn to move backward and forward by
imitating a cheetah. Since the isometry class of the optimal policy
\--- moving forward\--- of the cheetah and walker contains a
suboptimal element \---moving backward\---, we expect GWIL to recover
one of these two trajectories. Indeed, depending on the seed used,
GWIL produces a cheetah-imitating walker moving forward or a
cheetah-imitating walker moving backward, as shown in
\cref{fig:walker_result}.

\begin{figure}[t]
  \centering
  \begin{tikzpicture}
    \node[anchor=south west] at (0,0) {\includegraphics[width=1\textwidth]{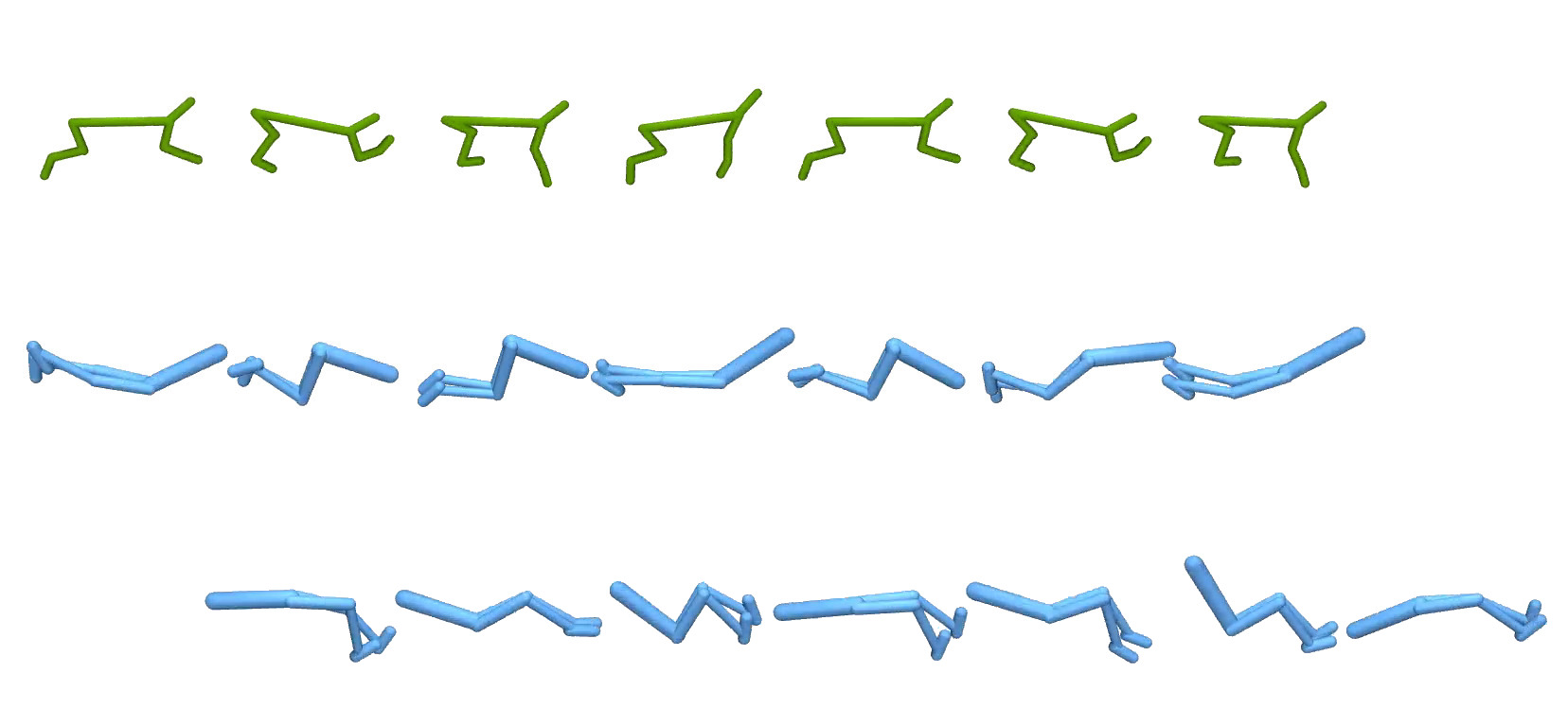}};
    \draw (0,0.7) -- (\textwidth,0.7);
    \draw (0,2.97) -- (\textwidth,2.97);
    \draw (0,4.94) -- (\textwidth,4.94);
    \draw[->] (12.2,5.4) -- (12.7,5.4);
    \draw[->] (12.5,3.4) -- (13,3.4);
    \draw[<-] (1,1.2) -- (1.5,1.2);
  \end{tikzpicture}
  \caption{
    Given a single expert trajectory in the cheetah's domain (above), GWIL recovers the two elements of the optimal policy's isometry class in the agent's domain (walker), moving forward which is optimal (middle) and moving backward which is suboptimal (below). Interestingly, the resulting walker behaves like a cheetah.
  }
  \label{fig:walker_result}
\end{figure}
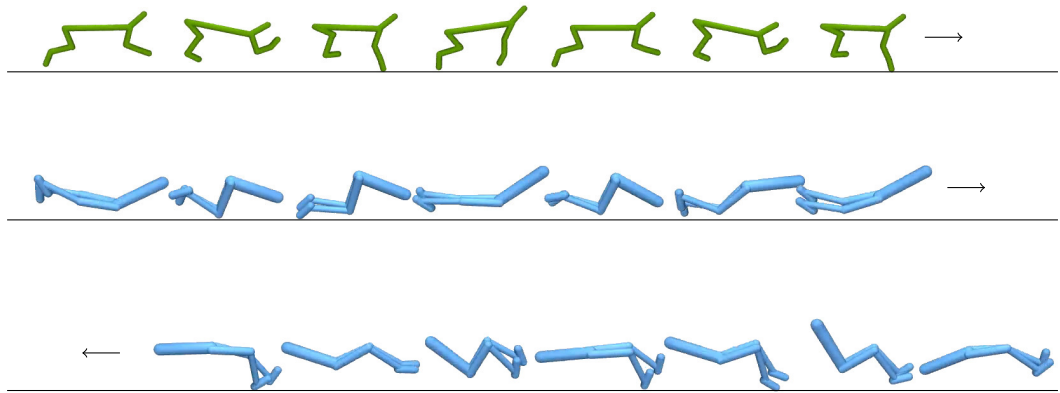

\section{Conclusion}
Our work demonstrates that optimal transport distances
are a useful foundational tool for
cross-domain imitation across incomparable spaces.
Future directions include exploring:
\begin{enumerate}
\item \textbf{Scaling to more complex environments and
  agents} towards the goal of transferring the structure
  of many high-dimensional demonstrations of complex tasks
  into an agent.
\item The use of $\mathcal{GW}$ to help agents
  \textbf{explore in extremely sparse-reward environments}
  when we have expert demonstrations available
  from other agents.
\item How $\mathcal{GW}$ compares
  to \textbf{other optimal transport distances} that
  work apply between two metric MDPs,
  such as \citet{otglobalinvariances}, that have more
  flexibility over how the spaces are connected and what
  invariances the coupling has.
\item \textbf{Metrics aware of the MDP's temporal structure}
  such as \citet{zhou2009canonical,Vayer2020TimeSA,cohen2021aligning}
  that build on dynamic time warping \citep{muller2007dynamic}.
  The Gromov-Wasserstein ignores the
  temporal information and ordering present within
  the trajectories.
\end{enumerate}

% \clearpage

% \subsection*{Acknowledgments}
% We would like to thank
% Ruihcao Jiang,
% Javad Tavakoli,
% and Yiqiang Zhao for useful
% corrections to our theorem statements.

\bibliography{refs}
\bibliographystyle{iclr2022_conference}

\newpage

\appendix
\section{Optimization of the proxy reward}
In this section we show that the proxy reward introduced in equation \ref{eq:rew} constitutes a learning signal that is easy to optimize using standards RL algorithms. Figure \ref{fig:gw_curves} shows proxy reward curves across 5 different seeds for the 3 environments. We observe that in each environment the SAC learner converges quickly and consistently to the asymptotic episodic return. Thus there is reason to think that the proxy reward introduced in equation \ref{eq:rew} will be similarly easy to optimize in other cross-domain imitation settings.
\begin{figure}[H]
  \centering
  \includegraphics[width=1\textwidth]{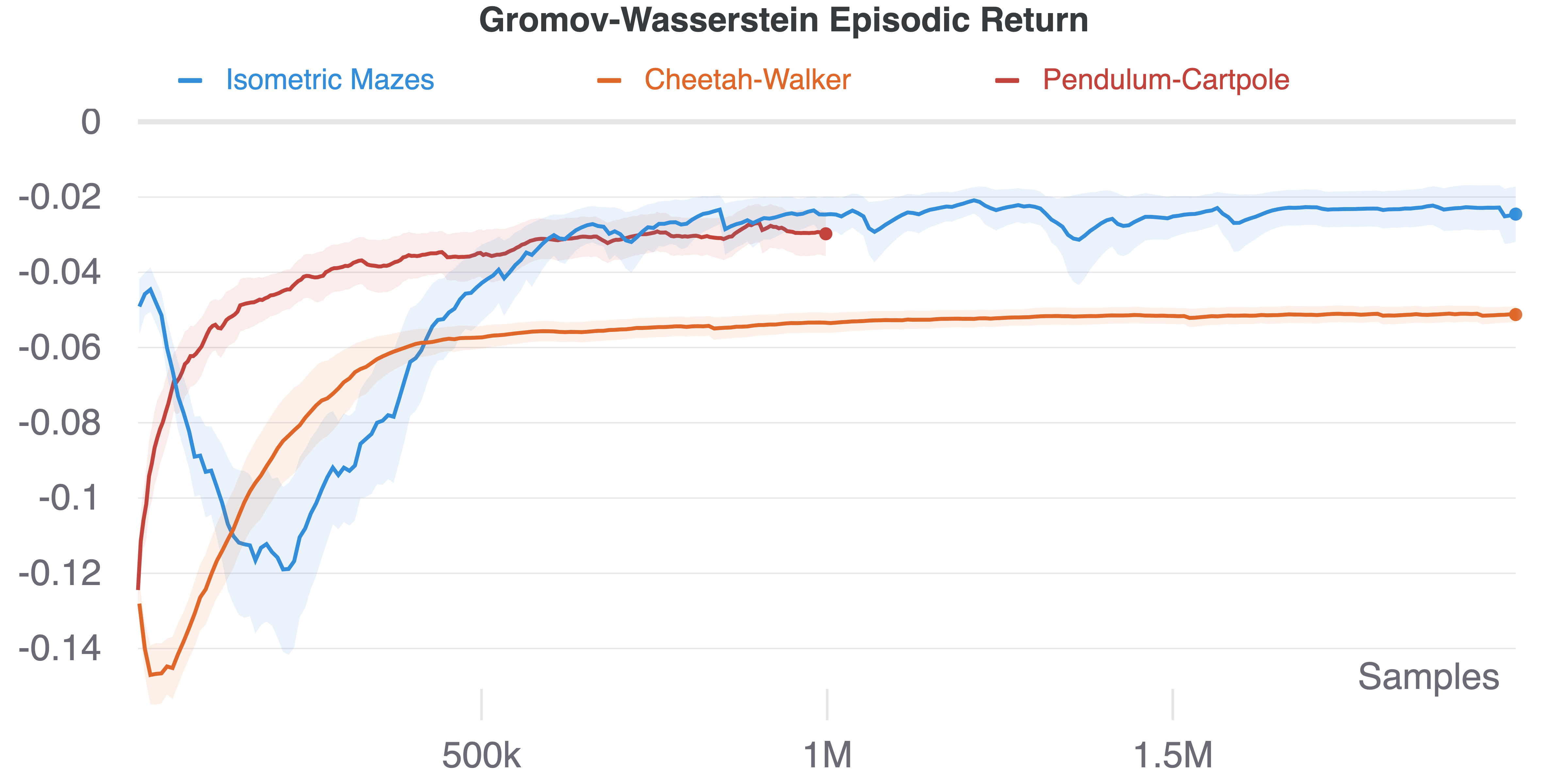}
  \caption{
    The proxy reward introduced in equation \ref{eq:rew} gives a learning signal that is easily optimized using a standard RL algorithm. 
  }
  \label{fig:gw_curves}
\end{figure}

\section{Transfer to sparse-reward environments}
In this section we show that GWIL can be used to facilitate learning in sparse-reward environments when the learner has only access to one expert demonstration from another domain. We compare GWIL to a baseline learner having access to a single demonstration from the same domain and minimizing the Wasserstein distance, as done in \cite{dadashi2020primal}. In these experiments, both agents are given a sparse reward signal in addition to their respective optimal transport proxy reward. We perform experiments in two sparse-reward environment. In the first environment, the agent controls a point mass in a maze and obtain a non-zero reward only if it reaches the end of the maze. In the second environment, which is a sparse version of cartpole, the agent controls a cartpole and obtains a non-zero reward only if he can maintain the cartpole up for 10 consecutive time steps. Note that a SAC agent fails to learn any meaningful behavior in both environments. Figure \ref{fig:sparse_maze_samples_gw_w} shows that GWIL is competitive with the baseline learner in the sparse maze environment even though GWIL has only access to a demonstration from another domain, while the baseline learner has access to a demonstration from the same domain. Thus there is reason to think that GWIL efficiently and reliably extracts useful information from the expert domain and hence should work well in other cross-domain imitation settings. 

\begin{figure}[H]
  \centering
  \includegraphics[width=1\textwidth]{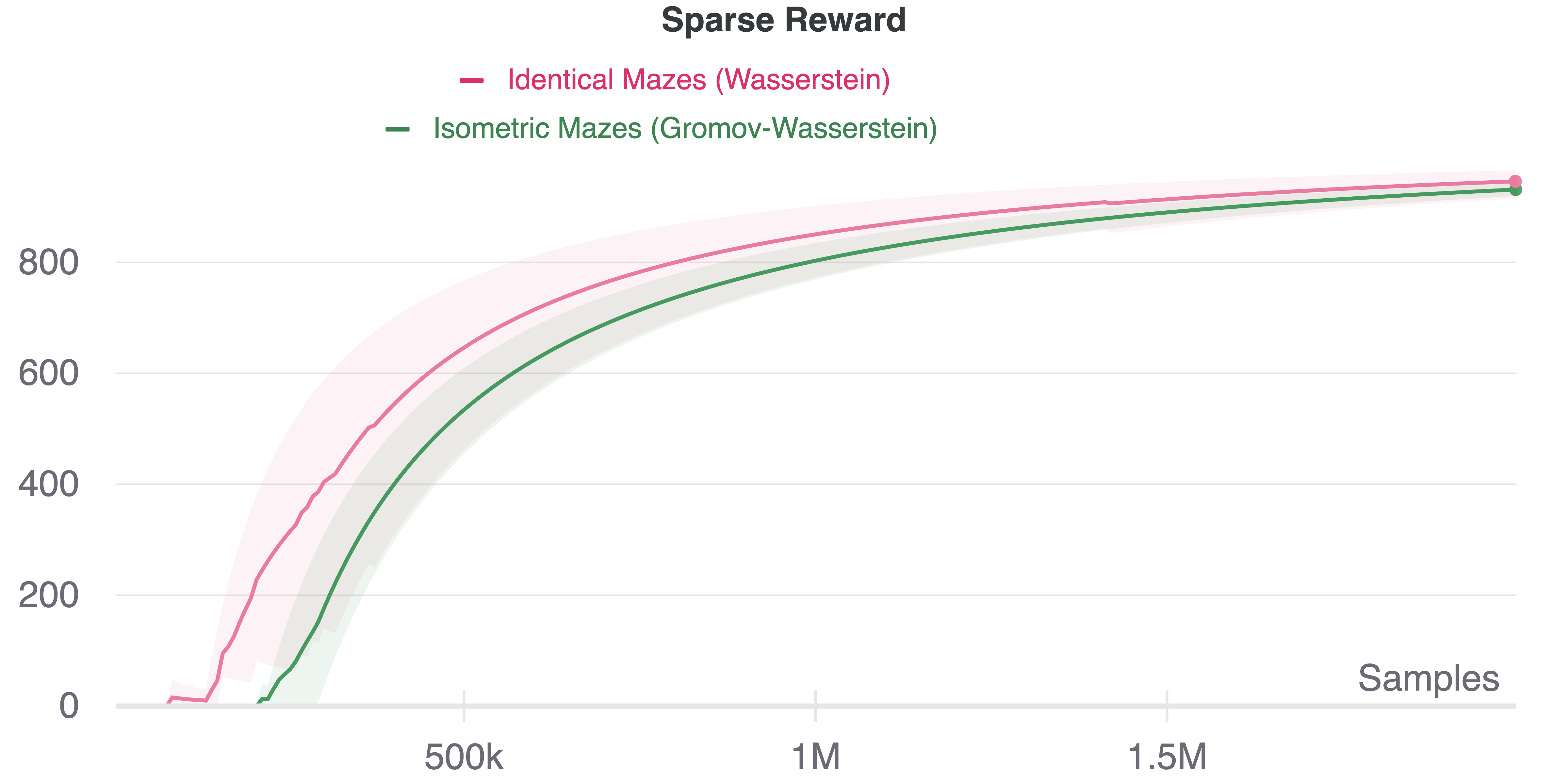}
  \includegraphics[width=1\textwidth]{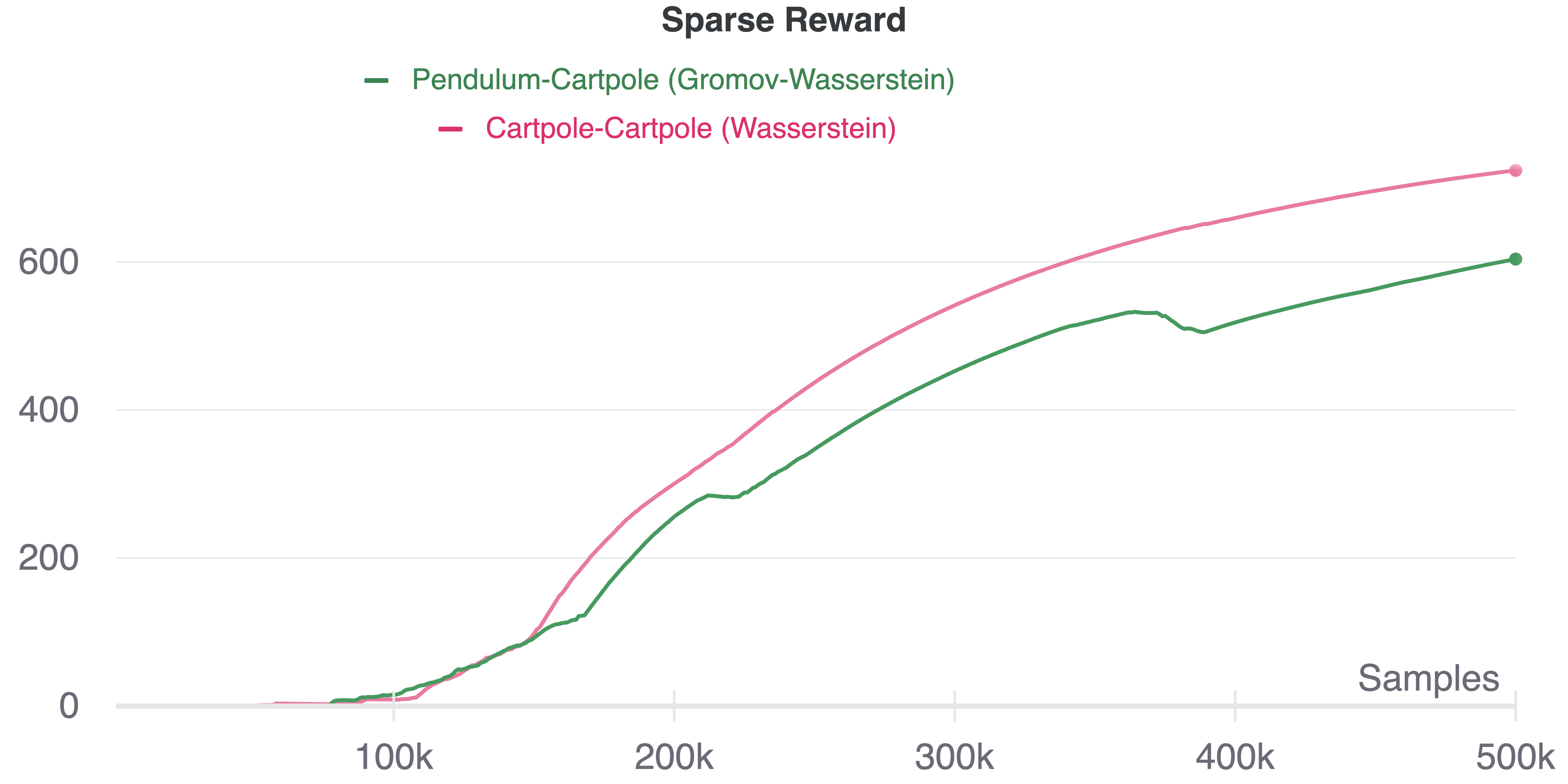}
  \caption{
    In sparse-reward environments, GWIL obtains similar performance than a baseline learner minimizing the Wasserstein distance to an expert in the same domain.     
  }
  \label{fig:sparse_maze_samples_gw_w}
\end{figure}

\section{Scalability of GWIL}
In this section we show that our implementation of GWIL offers good performance in terms of wall-clock time. Note that the bottleneck of our method is in the computation of the optimal coupling which only depends on the number of time steps in the trajectories, and not on the dimension of the expert and the agent. Hence our method naturally scales with the dimension of the problems. Furthermore, while we have not used any entropy regularizer in our experiments, entropy regularized methods have been introduced to enable Gromov-Wasserstein to scale to demanding machine learning tasks and can be easily incorporated into our code to further improve the scalability. Figure \ref{fig:sparse_maze_time_w_gw} compares the time taken by GWIL in the maze with the time taken by the baseline learner introduced in the previous section. It shows that imitating with Gromov-Wasserstein requires the same order of time than imitating with Wasserstein. Figure \ref{fig:hor_vel_walker_gw_sac} compares the wall-clock time taken by a walker imitating a cheetah using GWIL to reach a walking speed (i.e., a horizontal velocity of 1) and the wall-clock time taken by a SAC walker trained to run. It shows that a GWIL walker imitating a cheetah reaches a walking speed faster than a SAC agent trained to run. Even though the SAC agent is optimizing for standing in addition to running, it was not obvious that GWIL could compete with SAC in terms of wall-clock time. These results gives hope that GWIL has the potential to scale to more complex problems (possibly with an additional entropy regularizer) and be a useful way to learn by analogy.

\begin{figure}[H]
  \centering
  \includegraphics[width=1\textwidth]{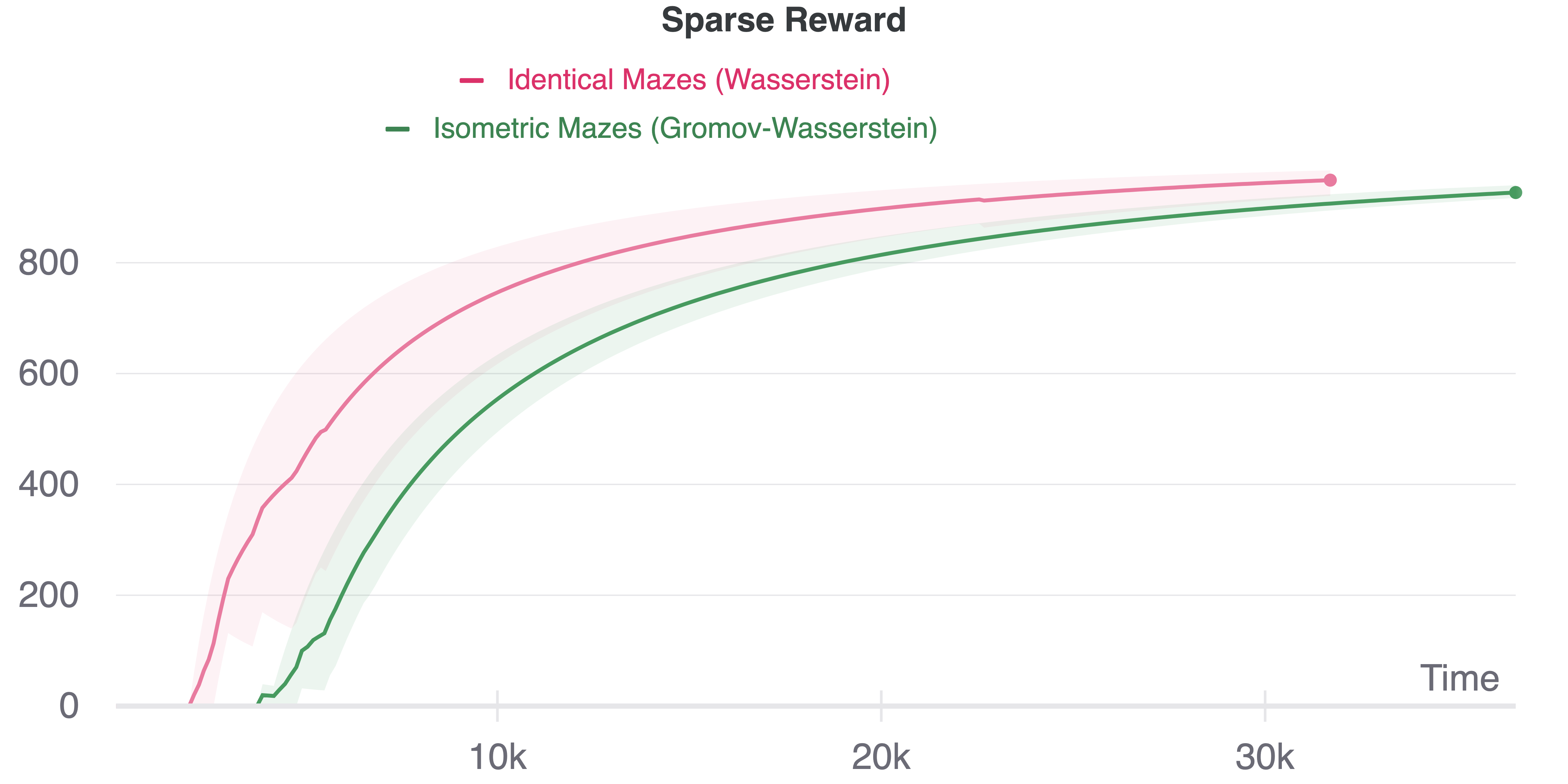}
  \caption{
    In the sparse maze environment, GWIL requieres the same order of wall-clock time than a baseline learner minimizing the Wasserstein distance to an expert in the same domain.     
  }
  \label{fig:sparse_maze_time_w_gw}
\end{figure}

\begin{figure}[H]
  \centering
  \includegraphics[width=1\textwidth]{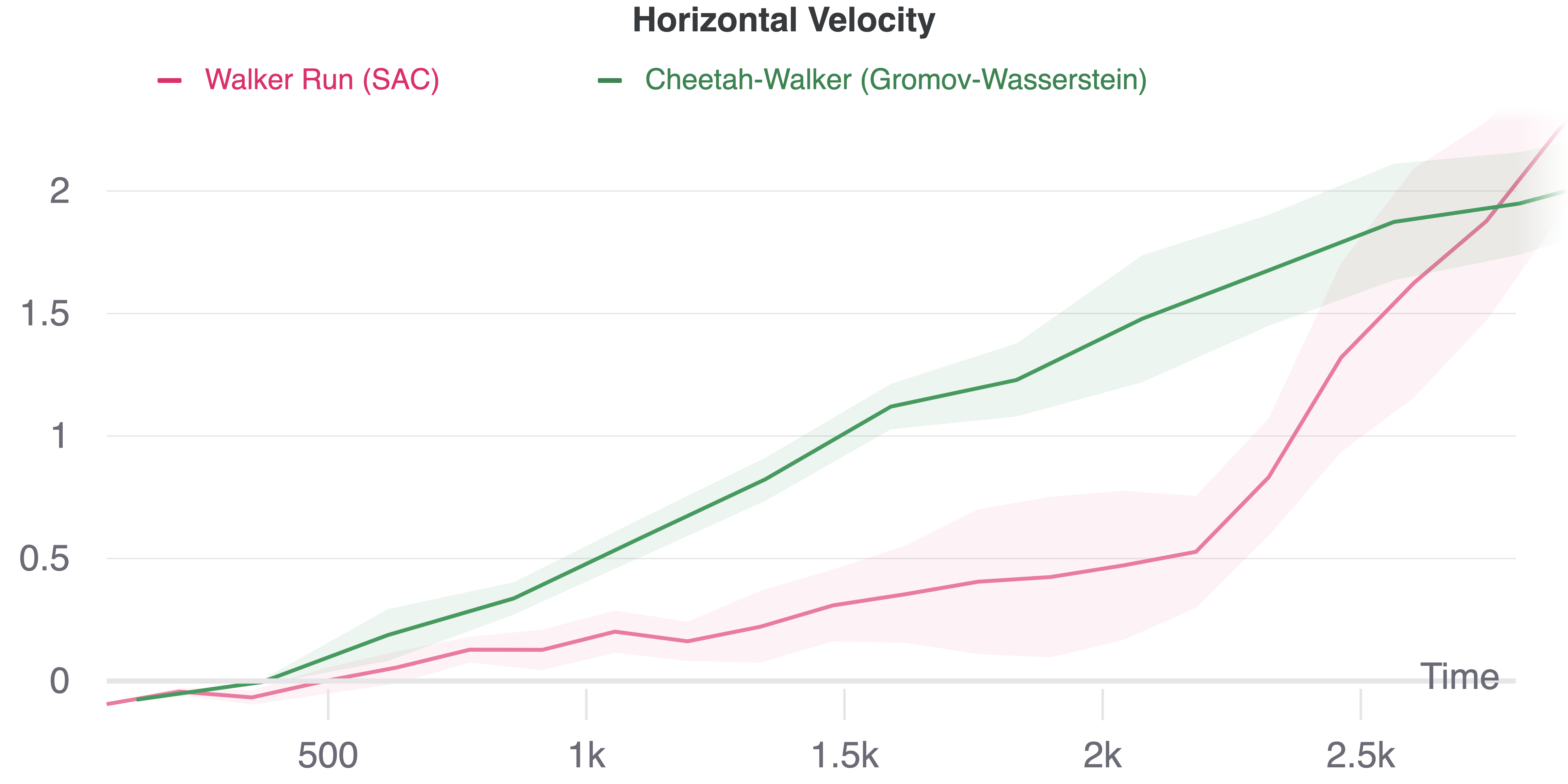}
  \caption{
    A GWIL walker imitating a cheetah reaches a walking speed faster than a SAC walker trained to run in terms of wall-clock time.
  }
  \label{fig:hor_vel_walker_gw_sac}
\end{figure}

\end{document}